\documentclass{article}

 \PassOptionsToPackage{numbers, compress}{natbib}
    
\usepackage{natbib}

\usepackage[utf8]{inputenc} % allow utf-8 input
\usepackage[T1]{fontenc}    % use 8-bit T1 fonts
\usepackage{url}            % simple URL typesetting
\usepackage{booktabs}       % professional-quality tables
\usepackage{amsfonts}       % blackboard math symbols
\usepackage{nicefrac}       % compact symbols for 1/2, etc.
\usepackage{microtype}      % microtypography
\usepackage{xcolor}         % colors
\usepackage{graphicx}
\usepackage{subcaption} 
\usepackage[utf8]{inputenc}

\usepackage[pdftex,colorlinks,linkcolor=blue,citecolor=blue,filecolor=blue,urlcolor=blue]{hyperref}

\usepackage[letterpaper,hmargin=1.0in,vmargin=1.0in]{geometry}
\usepackage{svg}

\usepackage[utf8]{inputenc} % allow utf-8 input
\usepackage[T1]{fontenc}    % use 8-bit T1 fonts

\usepackage{url}            % simple URL typesetting
\usepackage{booktabs}       % professional-quality tables
\usepackage{amsfonts}       % blackboard math symbols
\usepackage{nicefrac}       % compact symbols for 1/2, etc.
\usepackage{microtype}      % microtypography
\usepackage{xcolor}         % colors

\usepackage[american]{babel}
\usepackage[utf8]{inputenc} % allow utf-8 input
\usepackage[T1]{fontenc}    % use 8-bit T1 fonts
\usepackage{csquotes}
\usepackage{hyperref}       % hyperlinks
\usepackage{url}            % simple URL typesetting
\usepackage{booktabs}       % professional-quality tables
\usepackage{amsfonts}       % blackboard math symbols
\usepackage{nicefrac}       % compact symbols for 1/2, etc.
\usepackage{microtype}      % microtypography
\usepackage{xcolor}         % colors
\usepackage{array}          %nice tables
\usepackage{tikz}
\usetikzlibrary{calc,positioning,decorations.pathreplacing,shapes}
\usepackage{amsmath}
\usepackage{amsthm}
\usepackage{amssymb}
\PassOptionsToPackage{normalem}{ulem}
\usepackage{ulem}
\usepackage{dsfont}
\usepackage[noabbrev, capitalise]{cleveref}

\usepackage{enumitem}
\usepackage{xspace}
\usepackage{graphicx}
\usepackage{subcaption}

\theoremstyle{plain}
\newtheorem{theorem}{\protect\theoremname}[section]
  \theoremstyle{plain}
  
  \theoremstyle{plain}
  
  \theoremstyle{remark}
  
  \theoremstyle{plain}
  
  \theoremstyle{plain}
  
  \theoremstyle{plain}
  
  \theoremstyle{plain}

\usepackage{mathtools}

\usepackage{contour}
\usepackage[normalem]{ulem}

\crefname{assumption}{Assumption}{Assumption}

\crefname{example}{Example}{Example}

\crefname{definition}{Definition}{Definition}

\makeatother
\crefformat{equation}{(#2#1#3)}

\numberwithin{equation}{section}

\contourlength{0.8pt}

\renewcommand{\underline}[1]{%
  \uline{\phantom{#1}}%
  \llap{\contour{white}{#1}}%
}

\definecolor{OliveGreen}{rgb}{0,0.6,0}
\definecolor{JaumeBlue}{rgb}{0,0,0.6}

\usepackage{babel}
  \providecommand{\conjecturename}{Conjecture}
  \providecommand{\corollaryname}{Corollary}
  \providecommand{\lemmaname}{Lemma}
  \providecommand{\questionname}{Question}
  \providecommand{\propositionname}{Proposition}
  \providecommand{\remarkname}{Remark}
\providecommand{\theoremname}{Theorem}

\DeclareMathOperator\Tr{Tr}

\newtheorem*{theorem*}{Theorem}

\newcommand{\ours}{MixiT} % OR IRA (input-independent random attention)
\title{
Is Random Attention Sufficient for Sequence Modeling? 
\\Disentangling Trainable Components in the Transformer
}
\author{
  Yihe Dong\\
  Princeton University \\
  \href{mailto:ydong@princeton.edu}{\color{black}{ydong@princeton.edu}}
  \and
  Lorenzo Noci \\
  Princeton University \\
  ETH Zurich \\
  \href{mailto:lorenzo.noci@inf.ethz.ch}{\color{black}{lorenzo.noci@inf.ethz.ch}}
  \and
  Mikhail Khodak \\
  Princeton University \\
  \href{mailto:khodak@wisc.edu}{\color{black}{khodak@wisc.edu}}
  \vspace{10pt}
  \and
  Mufan Li \\
  Princeton University \\
  \href{mailto:mufan.li@outlook.com}{\color{black}{mufan.li@outlook.com}}
}

\date{}
\begin{document}

\maketitle
% !TEX root = main.tex
% \input{packages}

\begin{abstract}
The transformer architecture is central to the success of modern Large Language Models (LLMs), in part due to its surprising ability to perform a wide range of tasks -- including mathematical reasoning, memorization, and retrieval -- using only gradient-based learning on next-token prediction. 
While the core component of a transformer is the self-attention mechanism, we question how much, and which aspects, of the performance gains can be attributed to it.  To this end, we compare standard transformers to variants in which either the MLP layers or the attention weights are frozen at initialization.
Surprisingly, we find that attention with \textit{frozen} key and query weights is not only able to form induction heads, but can also perform competitively on language modeling. We formalize this by proving a new expressivity result for transformer models with frozen key and query weights. To further isolate the contribution of attention, we design \ours{}, an architecture with entirely random attention scores, with provably stable signal propagation that overcomes prior depth-wise scaling challenges in random transformers. We use the successes and failures of \ours{} to understand the role each transformer component plays, such as attention being largely responsible for in-context reasoning, and MLPs being responsible for, but collaborates with attention, on knowledge memorization.
%we find that \ours{} matches the performance of fully trained Transformers on various algorithmic tasks, especially those involving basic arithmetic or focusing heavily on memorization.
%For retrieval-based tasks, we observe that having input-dependent attention scores is consistently beneficial, while \ours{} underperforms. We attribute this failure to its inability to form specialized circuits such as induction heads -- a specific circuit known to be crucial for learning and exploiting repeating patterns in input sequences.
Our results suggest that the transformer architecture has a built-in inductive bias towards forming specialized circuits, as it does even without learnable attention weights.\footnote{Our code is publicly available at \url{https://github.com/princeton-pli/MixiT}.}
%Our results underscore the importance of architectural heterogeneity, where distinct components contribute complementary inductive biases crucial for solving different classes of tasks.

\end{abstract}
% !TEX root = main.tex

\section{Introduction}

%Transformer architectures are the backbone of state-of-the-art models across numerous machine learning domains but come with significant computational costs, such as the multi-head attention mechanism {\bf CITE TRANSFORMERS, FLASH ATTENTION}. As models grow increasingly larger, the associated computational demands have become prohibitively expensive for both training and inference. Static attention—using fixed attention patterns generated either randomly or through something like a DFT {\bf CITE FNET}—offers a promising direction for efficiency improvements. However, most efficient attention alternatives have focused on dynamic attention {\bf CITE EXAMPLES}.

Transformers \cite{vaswani2017attention} have rapidly become the workhorse architecture in modern machine‑learning systems, powering state‑of‑the‑art models in language, vision, and scientific domains \cite{dosovitskiy2020image, team2023gemini, guo2025deepseek}. Their success is typically attributed to the self‑attention mechanism, which allows every token to aggregate information from the entire sequence and has been linked to emergent abilities such as long‑range retrieval, algorithmic reasoning, and in‑context learning. Yet we lack a precise answer to a fundamental question:
which degrees of freedom inside the Transformer are truly necessary for these behaviours, and which can be simplified away without harming performance?

Prior work has probed the internals of trained Transformers. Studies of attention maps consistently report the emergence of induction heads that copy information forward and enable in‑context retrieval \cite{olsson2022context}. Another line of work has focused on whether and how Transformers' capabilities can emerge \cite{allen2024physics, ramesh2023compositional, jain2023mechanistically} by studying synthetic datasets with \citet{zhong2024algorithmic} showing that models with frozen self‑attention but trainable embeddings still solve many algorithmic tasks. These results hint that different parts of the architecture are responsible for modeling different tasks.

A standard Transformer block, however, contains several interacting components: an attention block, with learnable, input-dependent queries and keys, and MLP blocks composed of fully connected layers.
To exemplify the complexity of the interaction, notice that the attention weights can change their value both through changes in the queries/keys parameters, as well as as through the residual‑stream representations that feed those projections.
Hence, because all the components are trained together, it is hard to discern the contribution of each individual one towards solving a given task. 
%
% Central to our work, there are two ways in which the model can change an attention head during training:
% \begin{itemize}
%     \item Query/key projections that tailor token‑to‑token weights to the input,
%     \item Residual‑stream representations that feed those projections.
% \end{itemize}
%
% Disentangling the importance of (1) versus (2) is challenging because both are usually trained together.
%
%
In this work, we study the role of each Transformer's component by freezing different parts of the architecture to their values at initialization. In partiucular, we consider the following simplified variants:
\begin{itemize}
    \item \emph{Frozen‑QK}, preserves the conventional attention structure but freezes the query and key weight matrices, allowing only the value weight matrix to be learned. 
    \item \emph{Frozen-MLP}, where the weight matrices of the MLP block are frozen to their initial value.
    \item \emph{MixiT} (Mixing Transformer), a Transformer variant in which the attention sub‑layer is replaced by a fixed, randomly initialized mixing matrix. After initialisation the mixing weights are frozen, preventing the model from adapting interactions during training or conditioning them on input content. Learning therefore takes place exclusively in the embedding and MLP blocks and in residual branches.
\end{itemize}

\begin{figure}[!t]
    \centering
    \includegraphics[width=\linewidth]{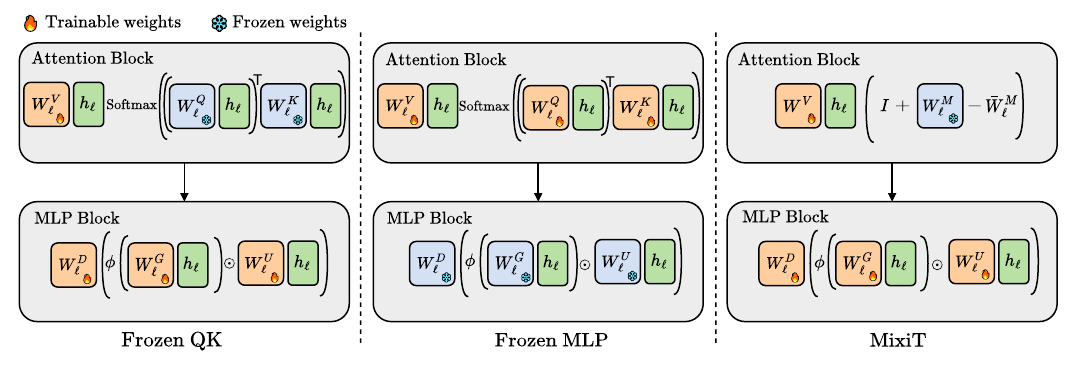}
    %\includesvg[width=\textwidth]{models_diagram.svg}
    \caption{Variants of the Llama Transformer model that we study.}
    \label{fig:architectures}
\end{figure}

Comparing these three simplified models with a fully‑trainable Transformer makes it possible to separate the role of (i) learned token‑dependent weights and (ii) the mere presence of a mixing operation, as well as (iii) understand the role of MLPs. 

We compare across several categories of tasks: mathematical reasoning, sentiment classification, memorization, information retrieval, hop$_k$, and language modeling. This allows us to inspect whether and how the different parts of the architecture are able to solve primitive tasks that relate to a range of basic reasoning and memorization skills.
Here we summarize our main contributions:
\begin{itemize}
    \item We present the surprising finding that Frozen-QK, which has random attention weights, can perform competitively with the standard transformer on language modeling tasks. Indeed, Frozen-QK develops the ability to form specialized circuits such as induction heads during training. This suggests that learnable attention weights are \textit{not} required to form specialized circuits.
    \item We analyze Frozen-QK's expressiveness, and show that it can indeed learn a wide class of sequence-level functions in Theorem~\ref{thm:frozen-qk}.
    \item We identify an explanation of the subspace selection hypothesis for random transformers \citep{zhong2024algorithmic}, namely, representation collapse, and connect it with the phenomenon that random models with standard initializations suffer degenerating performance with respect to depth, through an analysis of the covariance of hidden representations.
    %
    %\item We identify and analyze the \textit{obstruction} to scaling up in prior work that studied constant random attention, namely, representation collapse, and designed a principled way to fix it. 
    \item To resolve this obstruction towards scaling up model depth, we introduce a principled architecture -- \ours{}, as part of our spectrum of models, and use its successes at pure memorization tasks v.s. failures at pure in-context reasoning tasks such as k-hop, to derive conclusions on the roles of MLPs vs attention. Furthermore, we prove \ours{}'s training stability, by analyzing its covariance SDE, in Theorem~\ref{thm:mixit_cov_sde}.
    % - \ours{} is able to perform competitively on a variety of algorithmic tasks, and can be scaled up to large depths.
    % - We study MixiT's properties such as performance with respect to different numbers of attention heads. 
    \item We find that, in all but induction heads tasks and language modeling, \ours{} achieves performance comparable to the fully trained transformer and Frozen-QK. These results suggest that for a wide range of tasks, "attention mixing is all you need" regardless of its specific form, in the sense that learned query‑key interactions are not required. 
    %
    % \item When tested on tasks that require forming induction heads, \ours{} fails to achieve good performance, while both the fully trained transformer and Frozen-QK can solve them to close-to perfect accuracy. %We show that while the alternatives do form induction heads, \ours{} cannot do so by design, and this strongly correlates with the observed (in)ability to solve the task.
    %
    % \item On pure memorization tasks, Frozen-MLP loses two-thirds of the storage capacity of the fully-trainable Transformer, while Frozen-QK and \ours{} only a third. This suggests that the MLP layers are largely responsible for, but collaborates with attention, on memorization tasks. 
    %
\end{itemize}

Our results indicate that the Transformer architecture has some built-in inductive bias towards algorithmic abilities, namely the ability to form specialized circuits, as evidenced by the presence of induction heads even without learnable attention weights.
%Furthermore, while for a variety of tasks no particular attention pattern is needed, some learnable attention mechanism is required  when flexibly copying tokens. 
%As a consequence, due to the observed gap in language modeling, we speculate that the ability to solve associative recall and complex retrieval is a key ingredient to modern language models. 
More broadly, our results underpin the importance of having heterogeneous components that are more specialized in efficiently solving modular subtasks, enabling the model to express complex, wide-ranging behaviors.

%\yd{emphasize that the transformer architecture has some built-in inductive bias towards algorithmic abilities, namely the ability to form circuits}
% \yd{Add that perplexity for Frozen-QK is surprisingly close to standard.}

%\yd{Discuss why we should care }

%We investigate whether architectural modifications we collectively refer to as {\em shaping} can dramatically improve static attention performance {\bf CITE SHAPED TRANSFORMER}. By encouraging the stability of the representation covariance at initialization, shaping prevents rank collapse in many-layer models and allows for the stable training of novel attention variants without introducing an almost decades-worth of optimization tracks. This enables us to train our new architecture, the {\bf ShapedStaticTransformer}, across multiple training data and parameter scales.

%Empirically, we demonstrate that our approach {\bf DETAILED EMPIRICAL RESULTS ON SCALING, EFFICIENCY, AND DOWNSTREAM PERF}. Through ablation studies, we demonstrate how {\bf REASONING WITH FROZEN WEIGHTS / RANK COLLAPSE INVESTIGATION / SHAPED VS. UNSHAPED COMPARISONS}. 

%The success of shaping and static attention opens up a new direction for exploring the Pareto frontier of efficiency-perofrmance tradeoffs in resource-constrained scenarios, challenging the conventional wisdom that learnable attention is necessary for high-performance transformer models. Code to reproduce our results is available at: {\tt https://github.com/REPO-LINK}.

% !TEX root = main.tex

\section{Models Specification}
We consider a decoder-only Transformer based on the widely used Llama architecture \cite{touvron2023llama}. Given an input sequence $x \in \mathbb{R}^{V \times m}$, where $V$ is the vocabulary size and $m$ is the sequence length, embedded with a linear map to the hidden states $h_0 = W_{emb} \, x$ where $W_{emb} \in \mathbb{R}^{V \times n}$, where $n$ is the width of the model. At its core, the transformer uses $L$ stacked modules alternating the causal multi-head self-attention layers and MLP layers. Each self attention head is defined as:
\begin{equation}
\label{eq:attn_def}
    \text{Attn}(h_\ell) = W^v_\ell h_\ell  \, \text{Softmax}\left(\frac{1}{\sqrt{n_h}} Q_\ell^T K_\ell \right) \, ,
    \quad 
    Q_\ell = W_\ell^Q h_\ell \, ,
    \quad
    K_\ell = W_\ell^K h_{\ell} \, ,
\end{equation}
where $W^Q_\ell, W^K_\ell, W^v_\ell \in \mathbb{R}^{n_h \times n}$ are the queries' and keys' weights, and the outputs across multiple heads are concatenated.
The gated MLP layer is defined as:
% \yd{note concatenation on multiple heads}
%
\begin{equation}
\label{eq:mlp_def}
    \text{MLP}(h_\ell) = W^D_{\ell}\left(\phi(W_\ell^{G}h_\ell)\odot W_\ell^U h_\ell\right) \, .
\end{equation}
where $W_\ell^{G}, W_\ell^U \in \mathbb{R}^{{n_m \times n}}$ and $W^D_\ell \in \mathbb{R}^{{n \times n_m}}$. $n_h$ and $n_m$ are the dimensions of the queries/keys for each head and MLP projections. We also use skip connections in both blocks with a pre-normalization scheme \cite{xiong2020layer} and causal masking. We apply rotary embeddings \cite{su2024roformer} to the queries and keys of each layer. As in common practice, we use $n_{m} = 4n$ and $H n_h = n$.
\paragraph{Models with Frozen Weights.} 
In the Frozen-QK model, we set $W^Q_\ell, W^K_\ell$ to their value at initialization, and for the Frozen-MLP models we freeze $W^D_\ell, W^G_\ell, W^U_\ell$ for all layers. 
\paragraph{\ours{} -- Random Static Attention.}
We also design a model where the attention map itself is frozen and, to achieve that, input-independent. In the simplest case, this can be obtained by having a random matrix $M_\ell \in \mathbb{R}^{m \times m}$ entries with $\mathcal{N}(0, 1/m)$ entries, where the factor of $1/m$ acts as a variance-preserving normalizer. To ensure a stable forward pass in terms of depth and width scaling, we follow the principles of attention shaping \cite{noci2023shaped} and propose the following:
\begin{equation}
\label{eq:mixit}
    \text{Attn}(h_\ell) = W^v_\ell h_\ell \left(I + 
    W^M_\ell - \bar W^M_\ell \right) \,, 
    \quad 
    W^v_{\ell,ij} \overset{iid}{\sim} \mathcal{N}(0,\tfrac{1}{n}) \,, 
    W^M_{\ell,ij} \overset{iid}{\sim} \mathcal{N}(0,\tfrac{1}{\sqrt{nm}}) \, ,
\end{equation}
where $W^M_\ell$ is frozen at initialization and $\bar W^M_\ell$ contains the column-wise empirical average of $W^M_\ell$, to ensure that each row sums up to 1. In the Appendix, %\ref{sec:theory}
, we show that this attention has a stable forward pass, in the sense that the kernel of the activations has a well-defined depth-and-width limit, converging to a stochastic differential equation (SDE) \cite{li2022neural,noci2023shaped}. When we adopt this architecture, all the weights, excluding the random attention matrix, are trainable. 
Notably, the following convergence result implies the stability of the forward pass, in particular ruling out the numerical degeneracy such as rank collapse and vanishing gradients \citep{dong2021attention,noci2022signal}. 
%
% \yd{Write down theorem saying that this converges to stable SDE}
\\
\begin{theorem}[\ours{} Covariance SDE]
\label{thm:mixit_cov_sde}
Consider the \ours{} recursion $h_{\ell+1} = \text{Attn}(h_\ell)$ defined by \cref{eq:mixit} at initialization. 
Then as the width $n$ and depth $d$ go to infinity with $\frac{d}{n} \to \bar\tau > 0$, the upper triangular entries of the covariance matrix $\Phi_\ell = \frac{1}{n} h_\ell^\top h_\ell$ flattened to a vector in $\mathbb{R}^{m(m+1)/2}$ converges to the solution of the following SDE 
\begin{equation}
    d\Phi_\tau = \left[ \frac{1}{m} \Tr(\Phi_\tau) - M(\Phi_\tau) \right] \, d\tau + [\Sigma^v(\Phi_\tau) + \Sigma^M(\Phi_\tau)]^{1/2} \, dB_\tau \,, 
\end{equation}
where $M(\Phi) = \frac{1}{m^2} \sum_{\alpha\beta=1}^m \Phi^{\alpha\beta}$ is the average over all entries, $B_\tau$ is a standard Brownian motion in $\mathbb{R}^{m(m+1)/2}$, $\Sigma^v(\Phi)^{\alpha\beta,\gamma\delta} = \Phi^{\alpha\gamma} \Phi^{\beta\delta} + \Phi^{\alpha\delta} \Phi^{\beta\gamma}$ and 
\begin{equation}
    \Sigma^M(\Phi)^{\alpha\beta,\gamma\delta} = 
    \delta_{\alpha\gamma} C(\Phi^{\bullet\beta},\Phi^{\bullet\delta}) 
    + \delta_{\alpha\delta} C(\Phi^{\bullet\beta},\Phi^{\bullet\gamma}) 
    + \delta_{\gamma\beta} 
    C(\Phi^{\bullet\delta},\Phi^{\bullet\alpha}) 
    + \delta_{\beta\delta} 
    C(\Phi^{\bullet\alpha},\Phi^{\bullet\gamma}) \,, 
\end{equation}
where $\delta_{\alpha\gamma}$ is the Kronecker delta, 
$C(\Phi^{\bullet\beta},\Phi^{\bullet\delta}) = \frac{1}{m} \langle \Phi^{\bullet\beta}, \Phi^{\bullet\delta} \rangle - \overline \Phi^{\bullet\beta} \overline \Phi^{\bullet\delta}$, 
$\Phi^{\bullet\beta} = [ \Phi^{\alpha\beta} ]_{\alpha=1}^m$ is the $\beta$-th column vector, 
and $\overline \Phi^{\bullet\beta} = \frac{1}{m} \sum_{\alpha} \Phi^{\alpha\beta}$ is the average. 
\end{theorem}
The full proof can be found in the Appendix.

\paragraph{Positional embedding.}
As the random attention matrix $I + \frac{1}{\sqrt{mn}} W^M_\ell - \bar W^M_\ell$ in \ours{} does not depend on the input, the rotary positional embedding \citep{su2024roformer} standard in Llama models cannot be applied to \ours{}, as rotary embeddings are added to learned key and query embeddings, which are not present in \ours{}.
Hence, we implement a learnable positional embedding for each token position in the sequence, and add it to the corresponding token embedding in the first layer.

%\paragraph{Initialization and Training.} 
%All the hyperparameters and code are provided in the Appendix. %\ref{sec:exp_details} 
%and the supplementary material. The optimal hyperparameters are determined using a grid search specific for each task.

% \subsection{Random attention}
% \yd{divide into input-dependent and input-independent cases}
% Frozen-weight attention 

% \subsection{Final architecture}

% !TEX root = main.tex

\section{Experiments}

% \subsection{Scaling}

\subsection{Tasks}
We benchmark on a variety of tasks, covering categories including basic mathematical reasoning, memorization, sentiment classification, and language modeling.
The basic mathematical reasoning and memorization tasks are based on tasks used in \citep{zhong2024algorithmic}, with increased difficulty on some tasks to better reflect differences between architectures. More experimental details on these tasks can be found in Appendix~\ref{sec:exp_details}.

\paragraph{Decimal Addition.}
For the decimal addition task, the model learns to add two integers with the same number of digits. We randomly sample 50,000 pairs of ten-digit numbers, and train the model to predict their sum. The test set consists of 4,000 such sequences not in the training set. An example is $1234567890 + 2345678901 \to 3580246791$.

\paragraph{Needle in a Haystack (Retrieval).}

Each training instance encodes a small randomly generated sequence of pairs followed by a single query, and the model is required to emit the value associated with that query one step later.   
We uniformly sample a sequence length \(m \sim \mathcal{U}\{1,\dots,m_{\max}\}\), where $m_{\max}$ is the maximum sequence length.  
We then sample $m$ \emph{keys}  \(\{k_i\}_{i=1}^{m}\) iid from the set \(\{\frac{V}{2},\dots,V\}\) and $m$ \emph{values}  \(\{v_i\}_{i=1}^{m}\) from \(\{1,\dots,\tfrac{\text{V}}{2}-1\}\), where $V=256$ is the vocabulary size.  
The resulting keys are interleaved with their values to form the prefix $[(k_1,\,v_1),\, (k_2,\,v_2),\, \dots,\, (k_{m},\,v_{m})]$.
A query key \(k_q\) with $q\in [m]$ is chosen uniformly at random from the keys, and appended to the sequence. The goal is to predict the value corresponding to the query token. 
This task isolates retrieval ability, and probes associative recall. 
We sample 40,000 sequences for training, and 4,000 for testing. % We train with a batch size of $1024$ <--this was tuned.
Transformer-based models typically solve this task by forming induction heads \citep{olsson2022context}.

\paragraph{$k$-hop Induction Heads.}
Following \citep{sanford2024induction}, the $k$-hop induction heads task, or hop$_{k}$, recursively completes bigrams auto-regressively, 
by repeatedly predicting the token that followed the last occurrence of the currently-considered token.
As an example: given the input $X=a\textcolor{orange}{dc}\textcolor{blue}{ad}a$, the 2-hop induction heads prediction is $\textcolor{orange}{c}$.
\citep{sanford2024induction} showed that hop$_k$ is solvable by a $O(\log k)$-depth transformer.
Hence to achieve fair comparison across architectures, we fix the model depth at 5 layers, and search over other hyperparameters.

\paragraph{Modular Addition.}
This task evaluates the model's ability to perform addition, modulo a given prime $p$. In our case we sample 40,000 pairs $(a, b)$ of integers, each within the range $[1, p]$, for $p=599$, and train the model to predict $a + b \text{ mod } p$.

\paragraph{Parentheses Balancing (Dyck-1).}
The parentheses balancing task learns to predict whether a given parentheses sequence is balanced. Within any prefix in the sequence, the number of closing parentheses is less than or equal to the number opening parentheses. Hence this task is solvable by a linear time algorithm. We randomly sample 100,000 parentheses sequences of length 40 for training, while the test set consists of 4,000 such sequences not in the training set. An example is ``(()" $\to \text{False}$.

\paragraph{Memorization.}
We follow the procedure in \citet{zhong2024algorithmic}: we sample $512^2$ key-value pairs, where each key is independently sampled from its value, which is an integer in $[512]$. Because the key-value mapping is random, any success reflects the model’s ability to memorise arbitrary associations.   We measure success with the number of \emph{bits per parameter} that the model can store, defined as \texttt{$\# \text{total\_bits} \times \text{model\_acc} / \text{total\_trainable\_params}$}, where \texttt{$\text{model\_acc}$} is the training accuracy on the task. Notice that for this problem, storing one pattern requires $\log_2 512 = 9$ bits, thus \texttt{$\# \text{total\_bits} = 9 \cdot 512^2$}. 

\paragraph{Sentiment Classification.}
We use the Yelp polarity reviews dataset \citep{zhang2015character} to test each model variation's ability to predict a review's sentiment, i.e. whether a review is positive or negative.

\paragraph{Language Modeling.}
To test the model's ability to model natural language, we train the model to perform next-token-prediction.
We use two datasets, Wikitext-103 \citep{merity2016pointer} and Fineweb-edu \citep{penedo2024fineweb}. Wikitext-103 consists of 1,801,350 cleaned-up Wikipedia articles, with test set of size 4358. And Fineweb-edu consists of top-quality entries collected from web-crawled data, focusing on educational text. We randomly sample 1,048,576 entries for training, and 4096 to test.

\subsection{Model training}

For each task, we perform a grid search over a range of hyperparameters to train all model variations. The optimal hyperparameters are determined using a grid search specific for each task, which are detailed in the Appendix \ref{sec:exp_details}.
All model variants are trained on one to four H100 GPUs, depending on task complexity.

\section{Results}

% \paragraph{}
%
% \paragraph{Learnable attention is not required to form induction heads.}
\paragraph{Frozen-QK can solve induction heads tasks such as retrieval and hop$_k$.}
We test our spectrum of models on tasks that require forming induction heads \citep{elhage2021InductionHead}: needle-in-a-haystack retrieval \citep{olsson2022context} and $k$-hop induction heads (hop$_k$) \citep{sanford2024induction}. As shown in Table~\ref{tab:induction-head-tasks}, Frozen-QK is able to solve these tasks on par with the standard fully trained transformer model, indicating that induction heads can form even with frozen query and key weights.
On the other hand, \ours{} is unable to solve these tasks, due to its inability to form specialized circuits as the attention scores are frozen at random initialization.  

\begin{table}[h!]
\centering
\small
\begin{tabular}{lcccc}
\toprule
Task \textbackslash\ Model & Standard & Frozen-MLP  & Frozen-QK & \ours{}  \\
\midrule
$k$-hop Induction Heads $\uparrow$  & 99.99\% & 99.89\% & 96.73\% & 48.58\% \\
Retrieval   $\uparrow$ & 100\% & 100\% & 97.01\% & 11.24\%\\
\bottomrule
\end{tabular}
\vspace{10pt}
\caption{ 
Accuracies for tasks that require forming induction heads. 
Frozen-QK is able to solve the tasks on par with the standard fully trained transformer model, reinforcing the observation that induction heads can form even with frozen query and key projectors, e.g. as evident in Figure~\ref{fig:induction_head}.
Furthermore, as the attention module plays the key role in forming induction heads \citep{elhage2021InductionHead, crosbie-shutova-2025-induction}, Frozen-MLP, with its trainable attention modules, is able to perform to almost perfect accuracy.
On the other hand, \ours{} is unable to solve these tasks, due to its inability to form specialized circuits as the attention scores are frozen at random initialization.  
The retrieval results are for sequences of maximum number of key-value pairs $m_{\text{max}}=30$.}
\label{tab:induction-head-tasks}
\end{table}

More detailed results on the retrieval task are shown in \cref{fig:retrieval_accuracies}, where we test the performance of each architecture at varying task complexities, controlled by the maximum number of key-value pairs $m_{\max}$ in the sequence. 
%We search for the optimal learning rate for each architecture. To achieve a fair comparison across architectures, we fix the embedding dimension at $n=1024$, with $4$ heads and $2$ layers, which suffice to learn induction heads \citep{elhage2021InductionHead, edelman2024the}. 
As shown, Frozen-QK, Frozen-MLP, and the standard transformer all perform notably better than \ours{} as the retrieval complexity increases; and Frozen-QK eventually deteriorates in quality faster than Frozen-MLP and the standard transformer, underscoring the key role attention plays in forming induction heads.
% Frozen-QK has a drop in performances at $m_{\max} = 40$ and \ours{} struggles significantly after at $m_{\max} \geq 20$. These results provide compelling evidence that the MLP and embedding layers alone are insufficient when it comes to complex retrieval tasks.
% We perform an experiment using the needle in a haystack dataset where the degree of complexity is controlled by the maximal number of pairs $m_{\max}$. We fix the embedding dimension $n=1024$ divided across $4$ heads, and fix 2 layers, which are sufficient to solve this task by using induction heads. We search for the optimal learning rate in the range $[5\cdot 10^{-5}, 5\cdot10^{-3}]$. The results are shown in \cref{fig:retrieval_accuracies}. 

% Notice that while models with trainable queries and keys can solve the task across all the levels of difficulty, Frozen-QK has a drop in performances at $m_{\max} = 40$ and \ours{} struggles significantly after at $m_{\max} \geq 20$. These results provide compelling evidence that the MLP and embedding layers alone are insufficient when it comes to complex retrieval tasks. 

\begin{figure}[t!]
    \centering
    %\begin{subfigure}[b]{0.79\linewidth}
        \centering        \includegraphics[width=0.8\linewidth]{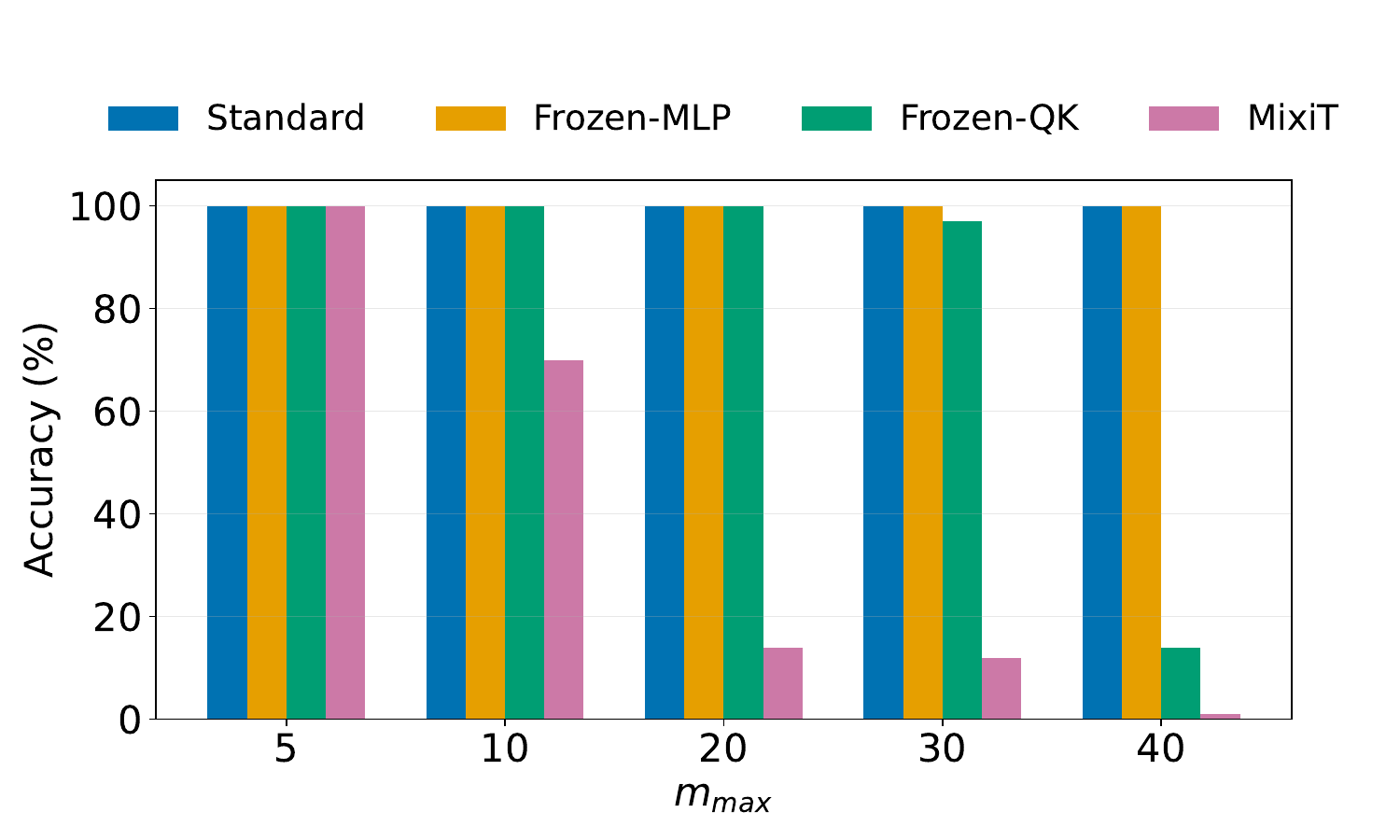}
        \caption{Retrieval accuracy as a function of the number of pairs in the
        sequence $m_{\max}$, which encodes task complexity.
        While \ours{} has rapidly deteriorating performance with respect to task complexity, Frozen-QK, with its ability to form induction heads, reaches its performance ceiling much more slowly.}        \label{fig:retrieval_accuracies}
    %\end{subfigure}
\end{figure}

% \begin{table}[h!]
% \centering
% \small
% \begin{tabular}{lccccc}
% \toprule
% Model \textbackslash $\,$ $m_{\max}$ & $\quad$ 5 $\quad$ & $\quad$ 10 $\quad$ & $\quad$ 20 $\quad$ & $\quad$ 30 $\quad$ & $\quad$ 40 $\quad$ \\
% \midrule
% Standard     & 100\% & 100\% & 100\% & 100\% & 100\% \\
% Frozen-MLP   & 100\% & 100\% & 100\% & 100\% & 100\% \\
% Frozen-QK    & 100\%  & 100\% & 100\% & 97\% & 14\% \\
% \ours{}      & 100\%  & 70\% & 14\% & 12\% & 1\% \\
% \bottomrule
% \end{tabular}
% \vspace{10pt}
% \caption{Retrieval accuracy as a function of the number of pairs in the sequence $m_{\max}$, which encodes task complexity. Notice that having trainable queries and keys is necessary to solve the task (first two models).}
% \label{tab:retrieval}
% \end{table}

\paragraph{Frozen-QK can perform competitively on language modeling.}
Surprisingly, Frozen-QK comes close to the standard transformer in terms of perplexity, as shown in Table~\ref{tab:language-modeling}.
This indicates that trainable attention weights are \textit{not} always required for successful language modeling. Indeed, as Figure~\ref{fig:induction_head} shows, specialized circuits such as induction heads can form even in Frozen-QK.
Less surprisingly, \ours{} lags behind the standard Transformer, supporting the hypothesis that input-dependent learned attention patterns, such as induction heads, are necessary for language modeling, corroborating earlier works \citep{olsson2022context, crosbie-shutova-2025-induction}.

%Theorem~\ref{thm:frozen-qk}
\begin{table}[h!]
\vspace{10pt}
\centering
\small
\begin{tabular}{lcc}
\toprule
Model \textbackslash\ Task &  Wikitext $\downarrow$ & Fineweb-edu $\downarrow$ \\
\midrule
Standard & 2.78 & 3.05 \\
Frozen-QK & 3.07 & 3.16 \\
\ours{} & 3.73 & 4.08 \\
\bottomrule
\end{tabular}
\vspace{10pt}
\caption{Performance on language modeling tasks, in terms of \textit{log} perplexity. 
Frozen-QK comes surprisingly close in performance to the standard Transformer, despite having random static query and attention weights. 
\ours{} has notably worse performance, supporting the hypothesis that input-dependent learned attention patterns, such as induction heads, are necessary for good language modeling.
}
\label{tab:language-modeling}
\end{table}

%\yd{Describe grid search range}
% \yd{Discuss these tables!}
\paragraph{Random static attention can perform certain algorithmic tasks.}
Table~\ref{tab:algo} contains the main results on algorithmic and sentiment classification tasks,
many with a large memorization focus. Both Frozen-QK and \ours{} are able solve such tasks. They are competitive with, and can even outperform, the standard fully trained transformer. These results highlight that input-dependent attention is not required for solving such algorithmic tasks.
\begin{table}[h!]
\centering
\small
\begin{tabular}{lccccr}
\toprule
Model \textbackslash\ Task & Decimal Addition$\uparrow$ & Dyck-1$\uparrow$ & Modular addition$\uparrow$ & Memorization$\uparrow$ & Yelp$\uparrow$  \\
\midrule
Standard & 98.58\% & 95.80\% & 100\% & 100\% & 90.55\% \\
Frozen-QK & 100\% & 97.38\% & 100\% & 100\% & 90.86\%  \\
\ours{} & 100\% & 96.17\% & 100\% & 100\% & 92.56\%  \\
\bottomrule
\end{tabular}
\vspace{10pt}
\caption{\ours{} performance on algorithmic and sentiment classification tasks. As shown, both Frozen-QK and \ours{} are able solve such tasks. They are competitive with, and can even be superior than, the standard transformer.
}
\label{tab:algo}
\end{table}

%\subsection{Characterizing random attention}
% \paragraph{Qualitative analysis of learned attention patterns}
\paragraph{MLPs are crucial, and collaborate with attention, on memorization.}
%
%We train an $L=2$, $n_h=4$ heads model on the memorization tasks for $10,000$ steps with a learning rate of $0.005$, for a consistent comparison across models. 
\cref{tab:mem} shows the accuracy and the storage capacity via \emph{bits per parameter}. We find that the standard transformer stores $2.98$ bits per parameters, which is slightly higher than in previous works \cite{allen2024physics, zhong2024algorithmic}. Most of the drop occurs in the Frozen-QK model, with $1.13$ bits per parameters, while Frozen-MLP and \ours{} have similar storage capabilities at $2.25$ and $2.18$, respectively. Note that these results are yielded in a setting where the accuracies are not saturated at 100\%, to give an accurate representation of bits per parameter. Hence they do not contradict the results in Table~\ref{tab:algo}.

\begin{table}[h!]
\centering
\small
\begin{tabular}{lccc}
\toprule
Model &  Memorization Accuracy $\uparrow$ & Bits Per Parameter $\uparrow$ & Trainable Parameters  \\
\midrule
Standard     & 100\% & 2.98 & 790400 \\
Frozen-MLP   & 19\% & 1.13 & 394880  \\
Frozen-QK    & 69\%  & 2.25 & 724352  \\
\ours{}      & 67\%  & 2.18 & 724736 \\
\bottomrule
\end{tabular}
\vspace{10pt}
\caption{ 
Standard Transformers outperform all the alternatives in terms of memorization capability, which suggests that MLPs and attention \textit{collaborate} to remember knowledge. This provides further evidence for recent findings such as knowledge circuits \citep{knowledgeCircuits} and query localization \citep{queryLocalization}.
Freezing the MLPs causes the most performance drop, indicating that they are the biggest factor when it comes to memorization.
Notice that \ours{} has slightly more parameters than Frozen-QK because of additional trainable positional embeddings.}
\label{tab:mem}
\end{table}

These results suggest that (1) the MLPs are largely responsible for memorization, however (2) there is a non-negligible additional contribution given by the integration of MLPs with learnable attention weights. 
This non-negligible additional contribution provides further evidence for recent findings such as knowledge circuits \citep{knowledgeCircuits} and query localization \citep{queryLocalization}, in that MLPs and attention \textit{collaborate} to remember knowledge. In particular, the disproportionately large increase in the learned bits per parameter from Frozen-QK to the fully trained transformer, from 2.25 to 2.98, suggests that the gain in accuracy is more than what can be accounted for by a mere increase in learnable parameters.

\paragraph{Performance with respect to number of heads for \ours{}.}
For some tasks, we observe that increasing the number of attention heads in \ours{} can notably improve performance, as demonstrated by the decimal addition and Dyck-1 parentheses balancing tasks in Table~\ref{tab:n_head}. 
Intuitively, since each attention head uses a different random attention matrix, more attention heads gives the learnable MLP components more diverse attention patterns to choose from based on the input, hence lessening the disadvantage of static attention.
Note that purely increasing the number of heads, without increasing the hidden dimension, reaches diminishing returns, as the per-head embedding dimension decreases proportionally, restricting expressiveness.

% Paths argument?
% \definecolor{mygreen}{HTML}{7CFC00}

\begin{table}[h!]
\centering
\small
\begin{tabular}{lcccc}
\toprule
Task \textbackslash\ Number of heads &  4 & 16 & 64 & 256 \\
\midrule
Decimal addition & 34.71\% & 34.70\% & 50.72\% & 91.87\%\\
Dyck-1 & 77.93\% & 81.68\% & 89.38\% & 91.83\%  \\
Yelp sentiment classification & 92.52\% & 92.56\% & 92.48\% & 91.72\% \\
\bottomrule
\end{tabular}
\vspace{10pt}
\caption{Accuracy with respect to the number of heads on various tasks for \ours{}. The hidden dimension is 512 for decimal addition and Dyck-1, and 1024 for Yelp sentiment classification. Increasing the number of heads increases the number of random attention matrices, giving learnable MLPs more diverse token mixing patterns to choose from based on the input, which can mitigate the disadvantage of static attention.
However, a positive performance correlation does not appear in all tasks, such as in Yelp. 
}
\label{tab:n_head}
\end{table}

However, we do not observe this performance boost consistently across tasks. For instance, for Yelp sentiment classification, the performance is invariant with respect to the number of heads. This phenomenon remains interesting work for future study.

%Performance with respect to depth. Show plot that our performance does not degrade wrt depth, whereas random transformer's does.

% !TEX root = main.tex

\section{Discussion}
%We discuss various aspects of \ours{}, Frozen-QK, and Frozen-MLP, in context of the standard transformer.

% \paragraph{Task separation with respect to model expressiveness}
\paragraph{Circuit learning and task separation.}
As the attention matrix in \ours{} is static and input-independent, \ours{} cannot adapt to each input and form specific circuits such as induction heads \citep{olsson2022context}. Induction head circuits look for recent occurrences of the current token, and attends to the tokens that follow with increased probability. This allows the standard transformer to easily adapt to the input context in language modeling. Hence, it is no surprise that \ours{} lags behind standard transformer for language modeling. It is perhaps more surprising that the perplexity comes close to that of the standard transformer.

The near-perfect performance on certain algorithmic tasks in Table~\ref{tab:algo} suggests that induction heads and other specialized input-dependent circuits are not required on these tasks. Hence the \ours{} performance on a given task can serve as a litmus test for whether in-context reasoning is required for that task.
For instance, \ours{} is able to perform well on the Yelp reviews dataset, despite the complexity of language used in reviews. This indicates that sentiment can be largely judged by the collective token embeddings, as opposed to next-token prediction in language modeling tasks, which requires retrieving specific details from the context, a task induction heads are apt at.

\paragraph{Learnable attention is not required to form induction heads.}
Interestingly, as shown in Figure~\ref{fig:induction_head} and demonstrated by its performance shown in Table~\ref{tab:induction-head-tasks} and Table~\ref{tab:language-modeling}, the Frozen-QK model can solve the retrieval task by forming induction heads. 

\begin{figure}[h!]
    \centering
    %\begin{subfigure}[]{0.79\linewidth}
        \centering
        \includegraphics[width=0.7\linewidth]{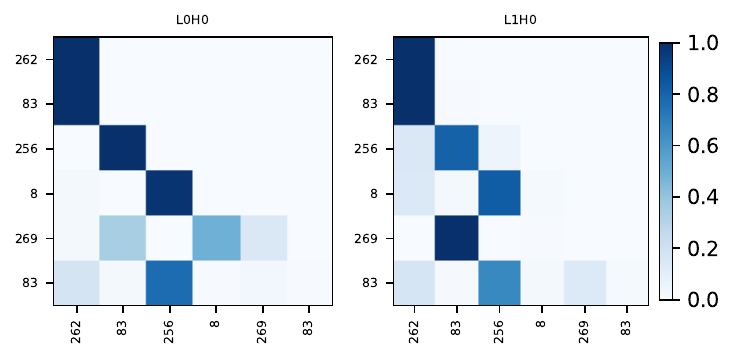}
        \caption{The Frozen-QK model can solve the retrieval task by forming an
        induction head. In the first head, each token attends to the previous
        one; in particular, the query token \texttt{83} is attended by
        \texttt{256}. In the head of the second layer, the correct token
        \texttt{256} is retrieved.}        \label{fig:induction_head}
    %\end{subfigure}
    %\label{fig:combined_retrieval_induction}
\end{figure}

These results naturally raises the question: how expressive is Frozen-QK? To answer this, we prove the following result showing Frozen-QK can approximate a wide class of functions:

\begin{theorem}[Universal Approximation of Frozen-QK] 
% Given sufficiently many attention heads and trainable parameters, Frozen-QK is capable of modeling the induction head function on sequences to arbitrary high precision.
% 
Every continuous causal function with compact support can be approximated arbitrarily well by one layer of multihead attention and MLP, where query and key weight matrices are frozen at random initialization. 
\label{thm:frozen-qk}
\end{theorem}

The proof can be found in Appendix \ref{sec:frozen-qk-proof}. In summary, the proof leverages the fact that standard transformers are universal approximators of such functions \citep{yun2020are}, and lifts this universal approximation to random feature transformers. 
Note that in practice, there are also the value weights and the MLP layers in the attention module, so we do expect the representation power to be even stronger, which is evidenced by empirical results. 

Furthermore, the proof sheds light on why MixiT cannot be a universal approximator: as $Q_k$ and $K_k$ are input-\textit{independent} in MixiT, the random feature $ g_i = x\ \text{Softmax} ( Q_k^\top K_k ) $ in MixiT is \textit{linear}. Hence non-linear functions in $x$, such as the induction head function, cannot be approximated.

\paragraph{Role of MLPs in knowledge storage.}
Previous works have highlighted the importance of MLPs in storing knowledge
\citep{dai2022knowledge, geva2023dissecting, geva2022transformer, geva2021transformer, yu2024locating, chughtai2024summing, meng2022locating, merullo2024language}. They posit that specific facts from the training data are stored in specific knowledge neurons. These works support our findings, in that MLPs are crucial in knowledge memorization.
However, our work does not prescribe knowledge localization, i.e. we don't attribute memorizing specific facts to specific neurons.

Our work adds characterization on knowledge memorization in more recent works on knowledge circuits \citep{knowledgeCircuits} and query localization \citep{queryLocalization}, where MLPs and attention are found to \textit{collaborate} on knowledge memorization, e.g. where attention selects the appropriate knowledge neurons depending on the query. Our work shows that even with static random attention weights, such as in Frozen-QK, attention and MLPs can still collaborate effectively, as evident in language modeling perplexities similar to that of the standard transformer, and the formation of specific input-dependent circuits. However, our results show that the role of MLPs in memorization outweighs that of attention, as evident by the fact that Frozen-MLP achieves much worse accuracy than Frozen-QK or \ours{} (\ref{tab:mem}). 

\paragraph{Relation to random transformers.}
\label{sec:random-transformers}

\citet{zhong2024algorithmic} studies the random transformer, wherein they train only the embedding and unembedding layers, and leave the intermediate layers fixed at random initialization. The random transformer was found to be able to perform nontrivial algorithmic tasks. 
There are several notable differences between \ours{} and the random transformer. The principal difference is that the random attention matrix $I + W^M_\ell - \bar W^M_\ell$ in \ours{} is \textit{input-independent}, and hence cannot develop specialized circuits that adapt to the input, whereas such circuits are possible for the random transformer. Moreover, the MLPs in \ours{} are trained, whereas they remain frozen in the random transformer. These differences allow us to study the roles different trained components play in the transformer architecture.

Furthermore, through in depth hyperparameter searches, we find that the random transformer does not scale well with respect to depth, confirming some of the original findings \citep{zhong2024algorithmic}. 
However, \ours{}, with its carefully initialized random attention matrix designed to preserve signal propagation, does scale with respect to depth, suggesting that the random transformer suffers from signal propagation challenges and rank collapse without appropriate shaping \citep{dong2021attention, noci2022signal, noci2023shaped}.

\begin{table}[h!]
\centering
\small
\begin{tabular}{lccc}
\toprule
Model \textbackslash\ Depth &  2 & 8 & 16 \\
\midrule
\ours{} & 100\% & 100\% & 100\% \\
Random Transformer & 100\% & 23.53\% & 22.88\% \\
\bottomrule
\end{tabular}
\vspace{10pt}
\caption{Performance comparison between \ours{} and the random transformer, with respect to number of layers, on the decimal addition task. The random transformer's performance does not scale well with respect to depth, whereas proper attention matrix shaping helps \ours{} scale with respect to depth.
}
\label{tab:depth_decay}
\end{table}

These observations provide an explanation for the subspace selection hypothesis in \citep{zhong2024algorithmic}, 
%which proposes that random transformer models learn representations that focus on specific subspaces. 
namely rank collapse with respect to model depth. Rank collapse refers to the phenomenon where the representations of different inputs become more and more similar as model depth increases, and has previously been studied in e.g. \citep{dong2021attention, noci2022signal, noci2023shaped}.
To substantiate this, we analyze the covariance between hidden representations for the language modeling task on Wikitext. This covariance is calculated between the last layer hidden representations within a sequence (higher covariance means the representations of different tokens are more similar), then averaged across sequences.

\begin{table}[h]
\centering
\small
\begin{tabular}{lccccc}
\toprule
\textbf{Model \textbackslash\ \# Layer } & \textbf{2 } & \textbf{4 } & \textbf{8 } & \textbf{16 } & \textbf{32 } \\
\midrule
Random Transformer & .00218 & .01450 & .02075 & .01946 & .05438 \\
MixiT              & .00088 & .00260 & .00210 & .00154 & .00104 \\
\bottomrule
\end{tabular}
\caption{Covariance between hidden representations of different tokens for models of different depths.}
\label{tab:covariance}
\end{table}

Table~\ref{tab:covariance} shows that the covariance for the Random Transformer steadily increases as the model size increases in depth, but remains steady for MixiT. % given its very deliberate initialization in Theorem~\ref{thm:mixit_cov_sde}. 
This directly implies representation collapse in the Random Transformer, where the representation of different tokens become increasingly similar as the model size grows in depth.

%These qualitative results are consistent with Table 5 in our work, where the performance for the Random Transformer degenerates with respect to depth, but not for MixiT.

This covariance degeneracy is especially catastrophic for language modeling, where deep models are required to achieve good performance. This helps explain why the Random Transformer struggles with language modeling and solving tasks when the model size increases beyond a certain number of layers \citep{zhong2024algorithmic}.

\paragraph{Implications for architecture design.}
Notably, some of these results strengthen the argument for empirical approaches for architecture design found in previous work \cite{poli2024mechanistic, carstensen2025frozen}. In particular, \citet{poli2024mechanistic} uses performances in various synthetic tasks to design powerful hybrid architectures.
Our results suggest that in certain circumstances specific architectural components are entirely responsible for some basic operations (e.g. information retrieval), while in others it is a more coordinated effort (e.g. memorization). 
%\yd{mention truncating MLPs to subspaces}
%Thus, overall this favours hybrid architectures with diverse basic components. %\yd{consider moving to related work maybe}

% !TEX root = main.tex

\section{Related work}

% \yd{I typically prefer to put related works at the end of the paper, to get to the main point more quickly.}

\paragraph{Static Attention.}
Several studies explore simplified Transformers with frozen or random components. Notably, random Transformers with fixed layers but trainable embeddings can solve many algorithmic tasks \cite{zhong2024algorithmic}. Similarly, replacing attention with fixed matrices - as in Synthesizer \cite{tay2021synthesizer} or FNet \cite{lee-thorp2021fnet} — retains competitive performance on certain benchmarks, suggesting learned attention is not always necessary. However, retrieval tasks often require flexible, input-dependent attention to form induction-like circuits. While their attention is also random, unlike our \ours{} model it is input-{\em dependent} and so does not isolate the specific tasks for which attention is and is not needed.
We relate our results to their work in detail in Section~\ref{sec:random-transformers}.
Other past work has studied properties of other randomly frozen or lightly trained models, e.g. convolutional networks~\citep{jarrett2009what,saxe2011on,arora2019on}, largely without focusing on specific tasks.
In addition, \citep{hassid2022attention} also studies static attention, but uses the learned attention matrices over a reference corpus, and hence is not completely data-free, and does not have a stable signal propagation guarantee as \ours{}. Indeed the completely data-free attention variants \citep{hassid2022attention} tested, without any reference corpus, perform poorly. Our work shows that such a variant's performance depends heavily on the task: on whether specialized circuits need to be formed.

%\citep{ben-artzy2024attend} And unlike Ben-Artzy and Schwartz, our work freezes the attention weights (in Frozen-QK) and the attention matrices (in MixiT), not hidden state representations.
 
\paragraph{Stable Signal Propagation.} 
While there is a long line of work studying signal propagation in deep neural networks \cite{schoenholz2016deep,poole2016exponential,lee2017deep,yang2017mean}, it was only more recently that \citet{martens2021rapid} introduced the concept that modifying activation functions can significantly improve stability of signal propagation, and leading to rapid training of large scale vision models without using normalization and skip connections \cite{zhang2022deep}. 
This was later understood to yield a stable scaling limit, which is characterized by an SDE of the covariance matrix \cite{li2022neural}. 
The covariance SDE framework is then used to understand how to design and shape general non-linearities like the Transformer self-attention \cite{noci2023shaped}, resolving the rank collapse issue \cite{dong2021attention,noci2022signal} and has shown strong performance despite a much simplified Transformer block \cite{he2023simplifying}. 
Our theoretical result \cref{thm:mixit_cov_sde} also follows from this framework. 
\paragraph{Modular Tasks.}
Several works have investigated the capabilities of Transformers in the classes of tasks analyzed here, including arithmetic \cite{nogueira2021investigating}. %\yd{Please conclude here}
The role of feedforward layers in memorization in the Transformer architecture has been studied in \citet{geva2020transformer} and their inductive biases and scaling properties have been scrutinized \cite{bachmann2023scaling}. In this context, our work shows their relevance in conjuction with trainable or fixed attention. Orthogonal to our work, memorization has also been studied to understand generalization in neural networks \cite{zhang2016understanding, arpit2017closer, anagnostidis2022curious}.

\paragraph{Mechanistic Interpretability.}
Other closely related work is in the mechanistic interpretability, which aims to understand LLMs via examining and modifying their internals~\citep{bereska2024mechanistic}.
Closely related is work related to identifying and understanding the behavior of induction heads~\citep{olsson2022context,edelman2024the, bietti2023birth} and in-context learning \cite{chan2022transformers}. 
Our work demonstrates that the performance separation between input-dependent and input-{\em independent} attention is largely driven by the latter's inability to form induction heads.
\citep{chen2024distributional} also studies different components of the transformer model, but from a model dynamics point of view, analyzing differences in gradients between components, and on synthetic tasks such as indirect object identification and factual recall. 
Finally, \citet{meng2022locating} which reverse-engineer how different Transformer components support behaviors like memorization, retrieval, and generalization.
\paragraph{Efficient Attention.}
A last area that has seen significant effort at understanding attention is that of efficient Transformers~\citep{tay2022efficient}.
While we study what happens when attention is replaced with a fixed input-independent matrix, this field has studied various useful aspects of the attention matrix such as attention sinks~\citep{xiao2024efficient} and compression~\citep{kim2024lexico}.
As our work demonstrates that for many tasks the full power of input-dependent attention is not needed, it may have its own implications for efficiency, e.g. by removing the need for the KV-cache.

\section{Conclusion}

In this work, we designed a spectrum of model architectures to systematically study the components of a transformer model. We found that, surprisingly, trainable attention is not required to form specialized circuits, and Frozen-QK can indeed perform well on language modeling. We also identified an obstruction towards scaling depth-wise in prior random models, and designed a principled remedy. Our work sheds important light on the functionalities of different model components, such as attention and MLPs, and shows how each is crucial for different tasks.

% we isolated the roles of the architecture’s trainable components by selectively freezing parameters and introducing \ours{}, whose attention maps are fixed and random. Fixed attention suffices for many algorithmic tasks. However, adaptive attention is essential for associative recall and language modelling, as \ours{} fails on long-range retrieval and yields higher perplexity. However, memorization resides in the MLPs, with attention contributing significantly less. Together, these results support a heterogeneity hypothesis: MLPs and learnable attention mechanism are required for the full behavioural range of modern large language models.

\paragraph{Limitations and future work.}
Our work studied basic algorithmic, mathematical reasoning, and language modeling tasks. It is an exciting future direction to extend our study to more complex tasks such as human reasoning, which demands a wide-ranging skillset that does not fall neatly on one side of the retrieval-memorization spectrum.
Future work might study hybrid training schedules, in which only a subset of architectural modules remain trainable -- or are gradually unfrozen -- which may strike an even better accuracy–efficiency trade-off. Finally, probing how these architectural choices interact with emerging interpretability and safety techniques constitutes an exciting avenue for further research.

% !TEX root = main.tex

\section*{Acknowledgments }

We are very grateful for the insightful discussions with Danqi Chen and Boris Hanin.

\bibliographystyle{plainnat}
\bibliography{infinite}

%%%%%%%%%%%%%%%%%%%%%%%%%%%%%%%%%%%%%%%%%%%%%%%%%%%%%%%%%%%%

\newpage
\appendix 
% !TEX root = main.tex

\section*{Appendix}
%\section*{Appendix to Attention Retrieves, MLP Memorizes, Disentangling Trainable Components in the Transformer}

% \subsection{Proof of Theorem 2.1}
% \yd{Theorem and proof.}

\section{Proof of \cref{thm:mixit_cov_sde}}

\begin{theorem*}[\ours{} Covariance SDE]
% \label{thm:mixit_cov_sde}
% 
Consider the \ours{} recursion $h_{\ell+1} = \text{Attn}(h_\ell)$ defined by \cref{eq:mixit} at initialization. 
Then as the width $n$ and depth $d$ go to infinity with $\frac{d}{n} \to \bar\tau > 0$, the upper triangular entries of the covariance matrix $\Phi_\ell = \frac{1}{n} h_\ell^\top h_\ell$ flattened to a vector in $\mathbb{R}^{m(m+1)/2}$ converges to the solution of the following SDE 
\begin{equation}
    d\Phi_\tau = \left[ \frac{1}{m} \Tr(\Phi_\tau) - M(\Phi_\tau) \right] \, d\tau + [\Sigma^v(\Phi_\tau) + \Sigma^M(\Phi_\tau)]^{1/2} \, dB_\tau \,, 
\end{equation}
where $M(\Phi) = \frac{1}{m^2} \sum_{\alpha\beta=1}^m \Phi^{\alpha\beta}$ is the average over all entries, $B_\tau$ is a standard Brownian motion in $\mathbb{R}^{m(m+1)/2}$, $\Sigma^v(\Phi)^{\alpha\beta,\gamma\delta} = \Phi^{\alpha\gamma} \Phi^{\beta\delta} + \Phi^{\alpha\delta} \Phi^{\beta\gamma}$ and 
\begin{equation}
    \Sigma^M(\Phi)^{\alpha\beta,\gamma\delta} = 
    \delta_{\alpha\gamma} C(\Phi^{\bullet\beta},\Phi^{\bullet\delta}) 
    + \delta_{\alpha\delta} C(\Phi^{\bullet\beta},\Phi^{\bullet\gamma}) 
    + \delta_{\gamma\beta} 
    C(\Phi^{\bullet\delta},\Phi^{\bullet\alpha}) 
    + \delta_{\beta\delta} 
    C(\Phi^{\bullet\alpha},\Phi^{\bullet\gamma}) \,, 
\end{equation}
where $\delta_{\alpha\gamma}$ is the Kronecker delta, 
$C(\Phi^{\bullet\beta},\Phi^{\bullet\delta}) = \frac{1}{m} \langle \Phi^{\bullet\beta}, \Phi^{\bullet\delta} \rangle - \overline \Phi^{\bullet\beta} \overline \Phi^{\bullet\delta}$, 
$\Phi^{\bullet\beta} = [ \Phi^{\alpha\beta} ]_{\alpha=1}^m$ is the $\beta$-th column vector, 
and $\overline \Phi^{\bullet\beta} = \frac{1}{m} \sum_{\alpha} \Phi^{\alpha\beta}$ is the average. 
\end{theorem*}
\begin{proof}
Firstly, we recall that based on \citet{li2022neural}, the linear network covariance matrix $\Phi_\ell = \frac{1}{n} h_\ell^\top h_\ell$ for the recursion $h_{\ell+1} = W^v_\ell h_\ell$ for $W^v_{\ell,ij} \sim \mathcal{N}(0, \tfrac{1}{n})$ satisfies the Markov chain 
\begin{equation}
    \Phi_{\ell+1} = \Phi_\ell + \frac{ \Sigma^v(\Phi_\ell)^{1/2} \xi_\ell }{\sqrt{n}} \,, 
\end{equation}
where $\xi_\ell$ is a zero mean and identity covariance random variable, and the diffusion coefficient is $\Sigma^v(\Phi)^{\alpha\beta,\gamma\delta} = \Phi^{\alpha\gamma} \Phi^{\beta\delta} + \Phi^{\alpha\delta} \Phi^{\beta\gamma}$. 
Therefore, it is sufficient to isolate the contribution of the mixing component alone, and we will add the effect of the two components. 

To this end, we consider the equivalent recursion 
\begin{equation}
    h_{\ell+1} = h_\ell \left( I_m + \tfrac{1}{\sqrt{nm}} (W^M_\ell - \bar W^M_\ell) \right) \,, 
\end{equation}
where we consider $W^M_{\ell,ij} \sim \mathcal{N}( 0, 1 )$ instead of $\mathcal{N}(0, \tfrac{1}{nm})$ due to the pre-factor, and $\bar W^M_{ij} = \tfrac{1}{m} \sum_{k=1}^m W^M_{kj}$ replaces each entry by its corresponding column average. 

Next, we will observe that $\Phi_\ell$ satisfies a straight forward recursion 
\begin{equation}
\begin{aligned}
    \Phi_{\ell+1} &= \frac{1}{n} h_{\ell+1}^\top h_{\ell+1} \\ 
    &= \left( I_m + \tfrac{1}{\sqrt{nm}} (W^M_\ell - \bar W^M_\ell) \right)^\top \Phi_\ell 
    \left( I_m + \tfrac{1}{\sqrt{nm}} (W^M_\ell - \bar W^M_\ell) \right) \\ 
    &= 
    \Phi_\ell + 
    \tfrac{1}{\sqrt{nm}} \left[ \left( W^M_\ell - \bar W^M_\ell \right)^\top \Phi_\ell 
    + \Phi_\ell \left( W^M_\ell - \bar W^M_\ell \right) \right] \\
    &\quad\quad + 
    \tfrac{1}{nm} \left( W^M_\ell - \bar W^M_\ell \right)^\top \Phi_\ell \left( W^M_\ell - \bar W^M_\ell \right) \,, 
\end{aligned}
\end{equation}
which naturally separates itself into the diffusion and drift components via the coefficient scale of $\tfrac{1}{\sqrt{nm}}$ and $\tfrac{1}{nm}$ respectively. 

We will compute the drift term next. 
Here we will drop some super and subscripts to reduce clutter, and write 
\begin{equation}
\begin{aligned}
    \sum_{\alpha,\beta=1}^m 
    \mathbb{E}_\ell ( W-\bar W )^{\alpha\gamma} \Phi^{\alpha\beta} ( W-\bar W )^{\beta\delta} 
    &= 
    \sum_{\alpha\beta} 
    \Phi^{\alpha\beta} 
    \mathbb{E}_\ell \left[ 
    W^{\alpha\gamma} W^{\beta\delta}
    - \frac{1}{m^2} \sum_{\mu\nu} W^{\mu\gamma} \bar W^{\nu\delta}
    \right] 
    \\ 
    &= 
    \sum_{\alpha\beta} 
    \Phi^{\alpha\beta} 
    ( \delta_{\alpha\beta} \delta_{\gamma\delta} - \frac{1}{m^2} \sum_{\mu\nu} \delta_{\mu\nu} \delta_{\gamma\delta} ) \\ 
    &= 
    \delta_{\gamma\delta} \sum_{\alpha\beta} 
    \Phi^{\alpha\beta} 
    \left(\delta_{\alpha\beta} - \frac{1}{m} \right) \,, 
\end{aligned}
\end{equation}
where $\mathbb{E}_\ell [ \,\cdot\, ] = \mathbb{E}[ \,\cdot\, | \mathcal{F}_\ell ] $ and $\mathcal{F}_\ell = \sigma( \{h_k\}_{k \leq \ell} )$, which translates to the final drift of 
\begin{equation}
    \frac{1}{n} \left( \frac{1}{m} \Tr(\Phi) - M_\Phi \right) I_n \,, 
\end{equation}
where $M_\Phi = \frac{1}{m^2} \sum_{\alpha\beta} \Phi^{\alpha\beta}$ is the average over all entries. 

To calculate a single entry of the diffusion coefficient $\Sigma(\Phi)^{\alpha\beta,\gamma\delta}$, we will write $\widetilde W = W - \bar W$ and compute 
\begin{equation}
\begin{aligned}
    &\Sigma^M(\Phi)^{\alpha\beta,\gamma\delta} \\ 
    &= \frac{1}{m} \sum_{\mu,\nu=1}^m 
    \mathbb{E}_\ell ( \widetilde W^{\mu\alpha} \Phi^{\mu\beta} + \Phi^{\alpha\mu} \widetilde W^{\mu\beta} ) 
    ( \widetilde W^{\nu\gamma} \Phi^{\nu\delta} + \Phi^{\gamma\nu} \widetilde W^{\nu\delta} ) \\ 
    &= \frac{1}{m}
    \sum_{\mu,\nu} 
    \mathbb{E}_\ell \left[ 
    \widetilde W^{\mu\alpha} \Phi^{\mu\beta} 
    \widetilde W^{\nu\gamma} \Phi^{\nu\delta} 
    +
    \widetilde W^{\mu\alpha} \Phi^{\mu\beta} 
    \Phi^{\gamma\nu} \widetilde W^{\nu\delta} 
    + 
    \Phi^{\alpha\mu} \widetilde W^{\mu\beta} 
    \widetilde W^{\nu\gamma} \Phi^{\nu\delta} 
    + 
    \Phi^{\alpha\mu} \widetilde W^{\mu\beta} 
    \Phi^{\gamma\nu} \widetilde W^{\nu\delta} 
    \right] \,. 
    % \\ 
    % 
\end{aligned}
\end{equation}
% 
% \mufan{double check indices above}
% 
% where $\Phi^{\alpha\bullet} = [ \Phi^{\alpha\beta} ]_{\beta=1}^m$ is the vectorized notation, and we abuse notation slightly to allow inner products for both $\Phi^{\alpha\bullet}$ columns and $\Phi^{\bullet\alpha}$ rows. 

At this point we focus on one term and compute 
\begin{equation}
\begin{aligned}
    \mathbb{E}_\ell \, \widetilde W^{\mu\alpha} \widetilde W^{\nu\beta} 
    &= \mathbb{E}_\ell \, ( W^{\mu\alpha} - \bar W^{\mu\alpha} ) ( W^{\nu\beta} - \bar W^{\nu\beta} ) \\ 
    &= \mathbb{E}_\ell \, ( W^{\mu\alpha} W^{\nu\beta} - W^{\mu\alpha} \bar W^{\nu\beta} - \bar W^{\mu\alpha} W^{\nu\beta} + \bar W^{\mu\alpha} \bar W^{\nu\beta} ) \,, 
    \end{aligned}
\end{equation}
$\delta_{\alpha\gamma}$ is the Kronecker delta, and we can separate further then compute 
\begin{equation}
\begin{aligned}
    \mathbb{E}_\ell \, W^{\mu\alpha} W^{\nu\beta} &= \delta_{\mu\nu} \delta_{\alpha\beta} \,, \\ 
    \mathbb{E}_\ell \, W^{\mu\alpha} \frac{1}{m} \sum_{\nu'=1}^m W^{\nu'\beta} 
    &= \frac{1}{m} \sum_{\nu'} \delta_{\mu\nu'} \delta_{\alpha\beta} 
    = \frac{1}{m} \delta_{\alpha\beta} \,, \\ 
    \mathbb{E}_\ell \, \bar W^{\mu\alpha} W^{\nu\beta} 
    &= \frac{1}{m} \delta_{\alpha\beta} \,, \\ 
    \mathbb{E}_\ell \, \bar W^{\mu\alpha} \bar W^{\nu\beta} 
    &= \frac{1}{m^2} \sum_{\mu',\nu'=1}^m \delta_{\mu'\nu'} \delta_{\alpha\beta} 
    = \frac{1}{m} \delta_{\alpha\beta} \,. 
\end{aligned}
\end{equation}

This implies 
\begin{equation}
    \mathbb{E}_\ell \, \widetilde W^{\mu\alpha} \widetilde W^{\nu\beta} 
    = 
    \delta_{\mu\nu} \delta_{\alpha\beta} 
    - \frac{1}{m} \delta_{\alpha\beta} 
    = ( \delta_{\mu\nu} - \tfrac{1}{m} ) \delta_{\alpha\beta} \,. 
\end{equation}

At this point, we return to calculating $\Sigma^M(\Phi)^{\alpha\beta,\gamma\delta}$ and write 
\begin{equation}
\begin{aligned}
    \Sigma^M(\Phi)^{\alpha\beta,\gamma\delta} 
    &= 
    \frac{1}{m} \sum_{\mu\nu} 
    ( \delta_{\mu\nu} - \tfrac{1}{m} ) \delta_{\alpha\gamma} \Phi^{\mu\beta} \Phi^{\nu\delta} 
    + 
    ( \delta_{\mu\nu} - \tfrac{1}{m} ) \delta_{\alpha\delta} \Phi^{\beta\mu} \Phi^{\gamma\nu} \\ 
    &\quad\quad + 
    ( \delta_{\mu\nu} - \tfrac{1}{m} ) \delta_{\gamma\beta} \Phi^{\mu\delta} \Phi^{\alpha\nu} 
    + 
    ( \delta_{\mu\nu} - \tfrac{1}{m} ) \delta_{\beta\delta} \Phi^{\alpha\mu} \Phi^{\gamma\nu} 
    \\ 
    &= \delta_{\alpha\gamma} C(\Phi^{\bullet\beta},\Phi^{\bullet\delta}) 
    + \delta_{\alpha\delta} C(\Phi^{\bullet\beta},\Phi^{\bullet\gamma}) 
    + \delta_{\gamma\beta} 
    C(\Phi^{\bullet\delta},\Phi^{\bullet\alpha}) 
    + \delta_{\beta\delta} 
    C(\Phi^{\bullet\alpha},\Phi^{\bullet\gamma}) \,, 
\end{aligned}
\end{equation}
where $C(\Phi^{\bullet\beta},\Phi^{\bullet\delta}) = \frac{1}{m} \langle \Phi^{\bullet\beta}, \Phi^{\bullet\delta} \rangle - \overline \Phi^{\bullet\beta} \overline \Phi^{\bullet\delta}$, 
$\Phi^{\bullet\beta} = [ \Phi^{\alpha\beta} ]_{\alpha=1}^m$ is the $\beta$-th column vector, 
and $\overline \Phi^{\bullet\beta} = \frac{1}{m} \sum_{\alpha} \Phi^{\alpha\beta}$ is the average. 

To complete the proof, we will invoke the Markov chain convergence to SDE results in the Skorohod topology, see for example \citet[Proposition A.6]{li2022neural}, which gives us the desired result. 

\end{proof}

\section{Proof of Theorem~\ref{thm:frozen-qk}} 
\label{sec:frozen-qk-proof}
\begin{theorem*}[Universal Approximation of Frozen-QK] 
% Given sufficiently many attention heads and trainable parameters, Frozen-QK is capable of modeling the induction head function on sequences to arbitrary high precision.
% 
Every continuous causal function with compact support can be approximated arbitrarily well by one layer of multihead attention and MLP, where query and key weight matrices are frozen at random initialization. 
\end{theorem*}

\begin{proof} 
% We use the universal approximation \citep{NeufeldSchmocker2023, RahimiRecht2007, RahimiRecht2008}
% of the entire sequence, viewed as a single (flattened) vector data point. Consequently, we can view each head’s Frozen-QK component as a *non-linear* random feature, i.e. for head $k$ we can have the random feature $f_k$:

% $$ f_k = x\ \text{Softmax} ( Q_k^\top K_k ), $$

% where $Q_k, K_k$ are i.i.d. random variables (but still a function of input $x$). Hence, doing linear regression on these random features as the number of attention heads approaches infinity, is equivalent to doing kernel regression with the covariance kernel $\Phi(x, x’) = \mathbb{E} \langle f_k(x), f_{k’}(x’) \rangle$, where the expectation is over the randomness of $Q, K$. Therefore, by the Universal Approximation Theorem for random feature models (Theorem 3.2 in \citep{NeufeldSchmocker2023}), given sufficiently many trainable parameters, the kernel regression with kernel $\Phi$ can approximate, to arbitrary precision, any function in the Banach space of functions over sequences that is not killed by $\Phi$. In particular, as the sequence function of the induction head does not live in the null space of $\Phi$, the kernel regression can indeed fit such a function with sufficiently many attention heads and trainable parameters. 
% 
% 
The proof will follow the universal approximation theory for random feature regression in Banach spaces by \citet{NeufeldSchmocker2023}, where we will set up the relevant Banach spaces. 

Let the space of input sequences of length $m$, where each element is a vector in $V \subset \mathbb{R}^n$, where $V$ is compact, be denoted by $U = V^m$.
The Banach space of continuous functions that map an input sequence from a compact set $K \subset U$ to an output sequence in $U$ is denoted by $(C(K, U), ||\cdot||_{\infty})$, i.e. equipped with the sup-norm. 
The subspace of continuous \textbf{causal} %(or non-anticipative) 
functions, which we can denote as $C_{\text{causal}}(K, U)$, consists of all functions $f \in C(K, U)$ that satisfy the following property:
For any given position $i \in \{1, \dots, m\}$ and any two input sequences $u=(u_1, \dots, u_m)$ and $v=(v_1, \dots, v_m)$ in the domain $K$:
\[
\text{If } u_j = v_j \text{ for all } j \le i, \text{ then } (f(u))_i = (f(v))_i
\]
This condition ensures that the output at position $i$ only depends on the input up to position $i$. Since this subspace is a closed linear subspace of the Banach space $C(K, U)$, it is itself a Banach space with the same supremum norm.

We define a random feature model based on a single multi-head attention layer with causal masking. Let the input be a sequence of hidden states $h \in \mathbb{R}^{n \times m}$, where $m$ is the sequence length and $n$ is the embedding dimension. The layer has $N_h$ attention heads.
Conceptually, each attention head $i \in \{1, \dots, N_h\}$ generates a single, matrix-valued \textbf{random feature}. This feature is the product of the hidden space $h$ and an attention pattern $\mathcal{A}_i$, which is a function that maps the input sequence $h$ to an $m \times m$ matrix.
The randomness for each feature $\mathcal{A}_i$ comes from a pair of \textbf{frozen weight matrices}, $(W^Q_i, W^K_i)$, where $W^Q_i, W^K_i \in \mathbb{R}^{n \times d_k}$. These are drawn independently for each head from a random distribution (e.g. $\mathcal{N}(0, 1)$) at initialization and are not trained.

The random feature (attention pattern) for head $i$ is defined as:
\begin{equation}
    g_i = h \, \mathcal{A}_i(h; W^Q_i, W^K_i) = h \, \text{Softmax}\left(\frac{(W^Q_i h)^\top (W^K_i h)}{\sqrt{d_k}} + M_{\text{causal}}\right) \,, 
\end{equation}
where $M_{\text{causal}}$ is the causal mask matrix that prevents attention to future positions.

By Theorem 3 in \citet{yun2020are}, Transformers are universal approximators for compactly supported sequence-to-sequence functions.
%%%
%Before invoking the results of \citet{NeufeldSchmocker2023}, we will also need the fact that Transformers are universal approximators deterministically, which was shown in \citet{yun2020are} for compactly supported sequence functions. 
% 
Furthermore, as \citet{NeufeldSchmocker2023} shows that universal approximation can be lifted from deterministic feature functions to random feature functions in Banach spaces \citep[Theorem 3.2]{NeufeldSchmocker2023}, we can lift this universal approximation to the random features $g_i$, to approximate any causal function $f \in C_{\text{causal}}(K, U)$ to arbitrary precision.

%Now by using the universal approximation theorem in Banach spaces \citep[Theorem 3.2]{NeufeldSchmocker2023},  the random features $g_i$ can approximate any causal function $f \in C_{\text{causal}}(K, U)$ arbitrarily well. 

Note that, given that both the value weight matrix $W^V_i$ of each head as well as the weights of the MLP layer are trainable, this is strictly more expressive than the random feature model we just constructed with $g_i$. 
Therefore one layer of self-attention and MLP must also be a universal approximator of causal functions, as desired. 

\end{proof}

\section{Additional Experimental Details}
\label{sec:exp_details}

\subsection{Hyperparameter selection}
We select for the optimal hyperparameters, shown in Table~\ref{tab:hyperparameter-best}, for each task using a grid search. The search ranges for each task are shown in Table~\ref{tab:hyperparameter-range}, determined a priori depending on task complexity, e.g. language modeling is inherently more complex than memorization.

\begin{table}[h!]
\centering
\small
\begin{tabular}{lccccr}
\toprule
Task \textbackslash\ Hyperparameter & \# Layers & Hidden dimension & \# Heads & Learning rate & Batch size \\
\midrule
Algorithmic & [2, 4, 8] & [512, 1024] & [4, 8, 16, 64] & [1e-3, 5e-4, 1e-4] & [128, 256, 512] \\
Retrieval & [2, 4, 8] & [512, 1024] & [4, 8, 16, 64] & [1e-3, 5e-4, 1e-4] & [256, 512, 1024] \\
hop$_k$ & [5] & [256, 512, 1024] & [8] & [1e-4, 5e-4] & [128, 256] \\
Memorization & [2, 4, 8] & [512, 1024] & [4, 8, 16, 64] & [1e-3, 5e-4, 1e-4] & [256, 512, 1024] \\
Yelp & [4, 8] & [512, 1024] & [4, 8, 16, 64] & [1e-3, 5e-4, 1e-4] & [256, 512, 1024] \\
Language modeling & [8, 12] & [512, 1024] & [4, 8, 16] & [1e-3, 5e-4, 1e-4] & [256, 512, 1024] \\
\bottomrule
\end{tabular}
\vspace{10pt}
\caption{Hyperparameter ranges used during grid search, for all architectures.
Algorithmic tasks include decimal addition, Dyck-1 parentheses balancing, and modular addition.}
\label{tab:hyperparameter-range}
\end{table}

Furthermore, we use a sequence length of 256 for Yelp sentiment classification and language modeling tasks. Language modeling tasks are trained for 40,000 steps.

For the $k$-hop induction heads task, following \citep{sanford2024induction}, we generate data using a sequence length of 100, and a maximum of $k=16$ hops. We use a training set of 100,000 samples, a test set of 100 samples, and train for 5,000 steps.

For a fair comparison across models on the memorization task (Table~\ref{tab:mem}), we use the same hidden dimension, with $L=2$, and $n_h=4$ heads. We train for $10,000$ steps with a learning rate of $0.005$. 
Similarly, for a fair comparison on the retrieval task, reported in Figure~\ref{fig:retrieval_accuracies},
we fix the embedding dimension at $n=1024$, with $4$ heads and $2$ layers, which suffice to learn induction heads \citep{elhage2021InductionHead, edelman2024the}. We search for the optimal learning rate for each architecture. Note that these two sets of results are meant to capture model performance when the accuracies are not saturated at 100\%, to give a meaningful comparison between models.

\begin{table}[h!]
\centering
\small
\begin{tabular}{lccccr}
\toprule
Task \textbackslash\ Hyperparameter & \# Layers & Hidden dimension & \# Heads & Learning rate & Batch size \\
\midrule
Decimal Addition & 8 & 512 & 64 & 1e-3 &  128 \\
Dyck-1 & 4 & 512 & 64 & 1e-3 &  512 \\
Modular addition & 2 & 512 & 32 & 1e-3 &  256 \\
Retrieval & 2 & 1024 & 4 & 1e-4 & 1024 \\
hop$_k$ & 5 & 512 & 8 & 1e-4 & 128 \\
Memorization & 2 & 1024 & 4 & 1e-3 & 256 \\
Yelp & 4 & 1024 & 16 & 5e-4 & 256 \\
Wikitext & 12 & 512 & 8 & 5e-4 & 512 \\
Fineweb-edu & 12 & 512 & 8 & 5e-4 & 512 \\
\bottomrule
\end{tabular}
\vspace{10pt}
\caption{Optimal hyperparameters selected for \ours{}.
}
\label{tab:hyperparameter-best}
\end{table}

\section{Further Discussion}
%%% Remove after submitting to neurips, incorporated into main paper.
This section contains further discussion on experiments and results.

\iffalse
\paragraph{Role of MLPs in knowledge storage.}
Previous works have highlighted the importance of MLPs in storing knowledge
\citep{dai2022knowledge, geva2023dissecting, geva2022transformer, geva2021transformer, yu2024locating, chughtai2024summing, meng2022locating, merullo2024language}. They posit that specific facts from the training data are stored in specific knowledge neurons. These works support our findings, in that MLPs are crucial in knowledge memorization.
However, our work does not prescribe knowledge localization, i.e. we don't attribute memorizing specific facts to specific neurons.

Our work adds characterization on knowledge memorization in more recent works on knowledge circuits \citep{knowledgeCircuits} and query localization \citep{queryLocalization}, where MLPs and attention are found to \textit{collaborate} on knowledge memorization, e.g. where attention selects the appropriate knowledge neurons depending on the query. Our work shows that even with static random attention weights, such as in Frozen-QK, attention and MLPs can still collaborate effectively, as evident in language modeling perplexities similar to that of the standard transformer, and the formation of specific input-dependent circuits. However, our results show that the role of MLPs in memorization outweighs that of attention, as evident by the fact that Frozen-MLP achieves much worse accuracy than Frozen-QK or \ours{} (\ref{tab:mem}). 
\fi

\paragraph{Improved throughput for \ours{} and Frozen-QK}

Leads to a notable improvement in training time for language modeling.
For instance, on the Fineweb-edu dataset, with the same architecture and hyperparameters on the same infrastructure, Frozen-QK trains 1267.1 samples per second on average, whereas the standard transformer trains 1022.8 samples per second on average, while achieving similar log perplexities, 3.05 for standard and 3.15 for Frozen-QK. This represents a 23.9\% improvement in training throughput, leading to a 23.9\% speedup in terms of wall clock time.
\ours{} trains even faster, with 1349.0 training samples per second, or a 32.0\% improvement in throughput. However, \ours{} comes with noticeable degradation in perplexity.

%\subsection{Training stability}
%\yd{Remove}
\iffalse
\begin{figure}[htbp]
  \centering
  % First sub‑figure
  \begin{subfigure}[b]{0.45\textwidth}
    \centering
    \includegraphics[width=\linewidth]{plots/train_grad_norm_decimal.png}
    \caption{Grad norm}
    \label{fig:loss}
  \end{subfigure}
  \hfill
  \begin{subfigure}[b]{0.45\textwidth}
    \centering
    \includegraphics[width=\linewidth]{plots/train_loss_decimal.png}
    \caption{train loss}
    \label{fig:grad_norm}
  \end{subfigure}
  \caption{Training stability}
  \label{fig:train-stability}
\end{figure}
\fi
% \input{notes}

\end{document}